\documentclass{article}

\usepackage[nonatbib, preprint]{neurips_2025}

\usepackage[utf8]{inputenc} % allow utf-8 input
\usepackage[T1]{fontenc}    % use 8-bit T1 fonts
\usepackage{hyperref}       % hyperlinks
\usepackage{url}            % simple URL typesetting
\usepackage{booktabs}       % professional-quality tables
\usepackage{amsfonts}       % blackboard math symbols
\usepackage{nicefrac}       % compact symbols for 1/2, etc.
\usepackage{microtype}      % microtypography
\usepackage{xcolor}         % colors
\usepackage{graphicx}
\usepackage{subcaption}
\usepackage{multirow}
\usepackage{multicol}
\usepackage{amsmath}
\usepackage{amssymb}
\usepackage{mathtools}
\usepackage{amsthm}
\usepackage{makecell}
\usepackage{wrapfig}

\newtheorem{proposition}{Proposition}          % Independent counter
\newtheorem{assumption}{Assumption}            % Independent counter
\usepackage[numbers]{natbib}  % 'numbers' ensures numerical citations

\title{Uniform Loss vs. Specialized Optimization: A Comparative Analysis in Multi-Task Learning}

\author{%
  Gabriel S. Gama\\
  University of São Paulo\\
  \texttt{gabriel\_gama@usp.br} \\
  \And
  Valdir Grassi Jr.\\
  University of São Paulo\\
  vgrassi@usp.br \\
}

\begin{document}

\maketitle

\begin{abstract}
    Specialized Multi-Task Optimizers (SMTOs) balance task learning in Multi-Task Learning by addressing issues like conflicting gradients and differing gradient norms, which hinder equal-weighted task training. However, recent critiques suggest that equally weighted tasks can achieve competitive results compared to SMTOs, arguing that previous SMTO results were influenced by poor hyperparameter optimization and lack of regularization. In this work, we evaluate these claims through an extensive empirical evaluation of SMTOs, including some of the latest methods, on more complex multi-task problems to clarify this behavior. More specifically, we start our analysis by evaluating all SMTOs on a simple MNIST problem to identify the promising optimizers and assess them further on progressively more complex multi-task problems. Our findings indicate that SMTOs perform well compared to uniform loss and that fixed weights can achieve competitive performance compared to SMTOs. Furthermore, we demonstrate why uniform loss performs similarly to SMTOs in some instances. The source code is available at \url{https://github.com/Gabriel-SGama/UnitScal_vs_SMTOs}.
\end{abstract}

\section{Introduction}

When tackling complex real-world challenges in machine learning, one must contemplate the utility of Multi-Task Learning (MTL) \cite{Caruana1997}. Scenarios such as scene understanding, multilingual translation, and predicting chemical properties require multiple outputs, possibly making them promising candidates for MTL approaches due to their inherent inductive bias \cite{baxter2000model}.

Classical methods approach this problem by utilizing one model for each task, but most MTL scenarios have the potential to share information across tasks. For example, autonomous driving relies on multiple pieces of information extracted from the same visual input, such as semantic segmentation, object detection, instance segmentation, and others. Therefore, sharing information across those tasks is desirable to increase performance.

To efficiently share feature information across tasks, various approaches have been proposed, including model architecture designs \cite{misra2016cross, gao2019nddr, liu2019end}, task grouping strategies \cite{standley2020tasks, fifty2021efficiently}, Meta-Learning techniques \cite{hospedales2021meta}, and the focus of this work, Specialized Multi-Task Optimizers (SMTOs) \cite{sener2018multi, navon2022nashmulti}.

SMTOs are implemented together with a shared encoder architecture, where one common encoder outputs a feature space that is shared with the specific decoders used for each task \cite{sener2018multi, navon2022nashmulti}. Consequently, this reduces the computational cost when compared to the individual models approach. This unified model can be optimized by minimizing the sum of the losses, known as uniform loss. However, the simultaneous learning of multiple tasks introduces challenges such as conflicting directions and significant gradient norm disparities \cite{liu2021towards, javaloy2021rotograd}. SMTOs aim to address these challenges by applying mathematical tools to jointly and equally optimize all objectives.

SMTOs can be classified into two types: gradient-based and loss-based. The former compute the gradients of each task and combine them to find a suitable optimization direction \cite{sener2018multi, chen2018gradnorm, katrutsa2020follow, chen2020graddrop, liu2021towards, javaloy2021rotograd, navon2022nashmulti, nakamura2022leveraging}. This process adds complexity and computational cost to the learning problem as multiple backpropagation steps are needed. The loss-based methods only take into consideration the loss values \cite{kendall2017multi, liu2022autolambda, lin2022rlw, liu2023famo, shen2024go4align}, reaching computational times comparable to those of the uniform loss \cite{achituve2024bayesianUncGradAggr, shen2024go4align}.

One notable issue in the SMTO field is the lack of a standardized evaluation framework and procedure, leading to concerns about reproducibility \cite{xin2022current}. Understandably, some works have begun questioning the efficacy of SMTOs \cite{kurin2022defense, xin2022current}. These critiques attribute reported gains in SMTO papers to sub-optimal hyperparameter tuning or insufficient regularization. They provide various experiments, including natural language processing (NLP), computer vision, and reinforcement learning, as well as a theoretical hypothesis on why SMTOs apply a regularization effect and a comparison to a properly optimized baseline.

In this context, we propose to address these concerns by hypothesizing that proper hyperparameter optimization is crucial for fair performance assessment of SMTOs. Nevertheless, we also emphasize the significance of considering the inherent complexity of the MTL balancing problem. Our contributions include:

\begin{itemize}
    \item A comprehensive evaluation of current SMTOs throughout Section~\ref{sec:exp}, including recent ones that have not been covered in previous critiques, using the MNIST dataset \cite{lecun1998gradientLenet} as a starting point to identify promising SMTOs, followed by other commonly used datasets such as Cityscapes \cite{Cordts2016Cityscapes} and Quantum Chemistry and Molecular Physics (QM9) \cite{wu2018qm9}.
    \item An investigation into why equally weighted tasks can match the results of SMTOs. More specifically, the two similar tasks scenario in Section~\ref{sec:two_similar_tasks}, and a complexity analysis in Section~\ref{sec:complexity}.
    \item A re-evaluation of the hypothesis brought up by \cite{xin2022current} that fixed weights can match the same results achieved by SMTOs in Section~\ref{sec:fixed_weights}.
\end{itemize}

\section{Related Work}

In recent years, MTL and SMTOs have witnessed substantial research activity. Notable works, such as \cite{kendall2017multi, sener2018multi}, have garnered interest in MTL due to their performance improvements compared to the uniform loss and single task baselines. This surge has also given rise to numerous variations and novel SMTO methods, each improving on or proposing solutions to unaddressed concerns of early SMTOs \cite{katrutsa2020follow, liu2021conflict, javaloy2021rotograd, nakamura2022leveraging}. 

However, despite the growing body of research, the SMTO field faces a significant challenge: the absence of a formal benchmark. Consequently, reported results may vary between articles \cite{sener2018multi, xin2022current, navon2022nashmulti}, raising concerns about the improvements. Specifically, recent critiques have revealed potential issues, suggesting that the reported performance gains from SMTOs stem from limited regularization \cite{kurin2022defense} and under-optimized baseline hyperparameters \cite{xin2022current}.

\citet{kurin2022defense} show that all evaluated SMTOs are linked to a larger solution space, under-optimization or stochastic behavior, all of which induce a regularization effect \cite{dietterich1995overfitting, keskar2016large, kleinberg2018alternative}. They also evaluated their hypothesis by adding proper regularization such as weight decay, batch normalization \cite{ioffe2015batch}, dropout \cite{srivastava2014dropout}, and early stopping to the uniform loss, referenced as Unitary Scalarization (Unit. Scal.) in this paper, and compared it to SMTOs. After the changes, they achieved competitive results to SMTOs at a lower computational cost. 

Similarly, \citet{xin2022current} argue that reported SMTO gains may be attributed to poor hyperparameter optimization, with particular interest in the impact of learning rates compared to random seeds. Their work demonstrates that exploring multiple fixed weight values can construct a Pareto curve, and SMTOs can only obtain one of the possible solutions. Due to computational restrictions, the discussion regarding fixed weights will be partially addressed as, to the best of our knowledge, there is no method to predict those fixed weights accurately, and grid search becomes impractical for a higher number of tasks.

Our work distinguishes itself in two key ways. Firstly, we consider most of the recent SMTOs in our evaluation \cite{liu2021conflict, javaloy2021rotograd, navon2022nashmulti, liu2022autolambda, nakamura2022leveraging, liu2023famo}. Secondly, we adopt some variations to the commonly used problems to introduce more complexity, as detailed in Section~\ref{sec:exp}. Our objective is to demonstrate that SMTOs, mainly the recent ones, provide significant improvements over Unit. Scal. as the complexity increases. Additionally, we aim to elucidate the scenarios in which Unit. Scal. can match SMTOs and provide a complete evaluation of these promising SMTOs in multiple challenging MTL configurations.

\section{Experimental Evaluation}
\label{sec:exp}
Although numerous works utilize similar datasets and training configurations, a broadly accepted evaluation framework remains elusive. We have chosen specific problems from a selection of SMTO papers to address this challenge. More specifically, we restrict our evaluation to the supervised learning setting using datasets such as MNIST \cite{lecun1998gradientLenet}, Cityscapes \cite{Cordts2016Cityscapes}, and QM9 \cite{wu2018qm9}. To address recent critiques to SMTOs \cite{kurin2022defense, xin2022current} and ensure robustness, the models were trained with at least one form of regularization, except for the QM9 dataset, and the parameters were optimized by applying grid search. We determine the best hyperparameter configuration by training once with each value (excluding Multi-MNIST, which was trained five times for each value) and selecting the best validation result. The best combination is trained at least two more times, and the metrics reported here are from the test split. We consider the epoch in which the model achieved the highest validation metric to mitigate overfitting.

For comparative evaluation of different SMTOs, we used two widely employed metrics, the mean relative percentage gain metric $\Delta_{\text{mtm}}$ and the mean rank $MR$ \cite{vandenhende2021mtlsurvey, navon2022nashmulti, liu2023famo}. $\Delta_{\text{mtm}}$ measures the improvement achieved by the MTL model over single-task references, and the $MR$ evaluates how equally all tasks were optimized. Since tasks like semantic segmentation and disparity estimation use multiple metrics for evaluation, we adapt the equations to account for that. For each set of metrics $M_i$ for task $i \in  \mathcal{T}$, the single task reference $M_{b, i, j} \; \mid j \in M_i$ is compared to its multi-task counterpart $M_{\text{smto}, i, j} \; \mid j \in M_i$, so $\Delta_{\text{mtm}}$ is given by:
\begin{equation}
\Delta_{\text{mtm}} = \frac{1}{N}\sum_{i=1}^{N} \frac{1}{M_i} \sum_{j=1}^{M_i} (-1)^{l_{i,j}} \frac{M_{\text{smto},i,j} - M_{b,i,j}}{M_{b,i,j}},
\end{equation}
where $l_{i,j}$ is $1$ if a lower value is better and $0$ otherwise. Likewise, the mean rank is given by:
\begin{equation}
MR = \frac{1}{N}\sum_{i=1}^{N} \frac{1}{M_i} \sum_{j=1}^{M_i}R_{i,j},
\end{equation}
where $R_{i,j}$ is the rank of the SMTO on task $i$ with respect to the metric $j$. As a result, both metrics compute the mean within the same task and then the mean of every task, ensuring a fair evaluation in a multi-task setting.

From a wide array of publicly available SMTOs, we selected a subset for evaluation: UW \cite{kendall2017multi}, MGDA \cite{sener2018multi}, PCGrad \cite{yu2020pcgrad}, GradDrop \cite{chen2020graddrop}, EDM \cite{katrutsa2020follow}, IMTL \cite{liu2021towards}, CAGrad \cite{liu2021conflict}, RotoGrad \cite{javaloy2021rotograd}, Nash-MTL \cite{navon2022nashmulti}, SI \cite{navon2022nashmulti}, RLW \cite{lin2022rlw}, Auto-Lambda \cite{liu2022autolambda}, CDTT \cite{nakamura2022leveraging} and FAMO \cite{liu2023famo}. We highlight that there are still some SMTOs that could be evaluated in future works \cite{guo2018dynamic, chen2018gradnorm, lin2023scaleinv, lin2023dualbalancing, achituve2024bayesianUncGradAggr, shen2024go4align}.

Including in Unit. Scal. and two RLW variations (Normal and Dirichlet distributions) as done in previous works \cite{kurin2022defense}, we assessed a total of 16 SMTOs. To reduce computational cost, the SMTOs considered in this paper are filtered according to their performance in a regression and classification problem. Even though evaluating an SMTO with just one example is not a definitive quality measure, this two-task problem tests the method's ability to handle varying gradient norms. This concern motivated several SMTOs \cite{chen2018gradnorm, katrutsa2020follow, liu2021towards, navon2022nashmulti}. Based on the results and contributions from each SMTO, we select the most promising ones and proceed to test them further.

In this section, the chosen SMTOs are first evaluated in the specific case of two similar task problems to demonstrate that, despite being an MTL problem, Unit. Scal. approximates the behavior of the SMTOs and achieves competitive results. We also consider some factors that dictate the need for SMTOs by comparing Unit. Scal. to one of the SMTOs while evaluating different configurations with increasing multi-task learning balancing complexity, using the Cityscapes dataset. All promising SMTOs are then evaluated in the most challenging format. Furthermore, we provide a comparison on the QM9 dataset. Finally, we evaluate the performance of fixed weights. All hyperparameter combinations, SMTOs-specific hyperparameter optimization, and their best configuration are summarized in Appendix~\ref{app:hyper} for reproducibility.

\subsection{Selecting promising SMTOs for further evaluation}
\label{sec:mnist_filter}

In this paper, we tackle multiple MTL challenges to evaluate the performance of many SMTOs. However, due to the multitude of combinations and the necessity for hyperparameter optimization to ensure a fair evaluation \cite{kurin2022defense, xin2022current}, we apply a preliminary filter for SMTO selection based on the performance on the Multi-MNIST dataset \cite{sener2018multi}.

The Multi-MNIST dataset is a variation of the MNIST dataset for MTL where the digits are overlapped in the same image, one on the top left and one on the bottom right. This combined image is then used to predict information regarding the left and right digits. The most common variation is predicting the class of the left and right digits, as done in multiple works \cite{sener2018multi,yu2020pcgrad, javaloy2021rotograd, kurin2022defense, xin2022current}. Instead, we opt for the alternative suggested by \citet{nakamura2022leveraging}, where the model is trained to classify the top left digit and to reconstruct the bottom right digit. This choice was made because this variation adds more complexity, involving tasks with varying levels of difficulty and loss functions that result in gradients with different norms, an important factor when considering the specific case of two tasks, as shown in Section~\ref{sec:two_similar_tasks}. Those tasks will be referenced as CL, CR, RL, and RR for the classification (C) or reconstruction (R) of the left (L) or right (R) digit. When tasks are represented using the notation \{task 1\}\_\{task 2\}...\{task N\}, it indicates a multi-task scenario that simultaneously addresses tasks 1 through N.

For our model architecture, we use a variation of the Lenet model \cite{lecun1998gradientLenet} as \citet{nakamura2022leveraging}. We then add dropout layers to the encoder and decoder architectures for regularization. The hyperparameters are optimized by grid search, varying the learning rate $lr \in$ \{0.01, 0.075, 0.005, 0.0025, 0.001, 0.00075, 0.0005\} and dropout $p \in$ \{0.0, 0.1, 0.2, 0.3, 0.4, 0.5\}. Specific SMTO hyperparameters are also included in the search if necessary. We train each configuration five times and the best validation configuration five more times, giving a total of ten models for each SMTO.

\begin{figure}[ht]
  \centering
  \includegraphics[width=0.85\textwidth]{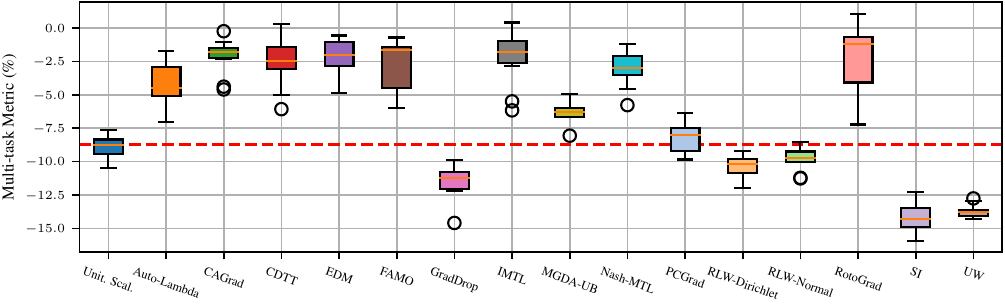}
  \caption{Box plot illustrating $\Delta_{\text{mtm}}$ scores on the Multi-MNIST classification and regression problem. The red line indicates the median value for Unit. Scal.. Notably, multiple SMTOs obtained substantial improvements over the Unit. Scal. baseline.}
  \label{fig:mnist_box_plot_CL_RR}
\end{figure}

The results, as presented in Figure~\ref{fig:mnist_box_plot_CL_RR}, demonstrate that Unit. Scal. achieves competitive results compared to some SMTOs. However, the majority outperform it significantly, primarily recent ones. Based on these results, we select EDM \cite{katrutsa2020follow}, IMTL \cite{liu2021towards}, CAGrad \cite{liu2021conflict}, RotoGrad \cite{javaloy2021rotograd}, Nash-MTL \cite{navon2022nashmulti}, Auto-Lambda \cite{liu2022autolambda}, CDTT \cite{nakamura2022leveraging} and FAMO \cite{liu2023famo} as promising SMTOs for further investigation.

Despite the Multi-MNIST dataset's simplicity, it effectively demonstrated the selected SMTOs' ability to provide a meaningful gradient direction even with varying gradient norms, enabling their distinction from other methods. Specifically, PCGrad and MGDA are known to emphasize gradients with the smallest and largest magnitudes, respectively \cite{liu2021towards}. GradDrop randomly drops gradients based on sign-purity, without considering magnitude balancing. Furthermore, methods like RLW, UW, and SI operate directly on loss values without estimating gradients, unlike FAMO and Auto-Lambda, potentially making them sensitive to differences in gradient magnitude. Ultimately, most of the selected SMTOs significantly outperformed Unit. Scal. in our evaluations.

\subsection{Two similar tasks}
\label{sec:two_similar_tasks}
The two-task setting is commonly used as a computationally efficient way of evaluating SMTOs, widely adopted in multiple MTL papers \cite{sener2018multi, kurin2022defense, xin2022current}. However, when considering the two-task setting, it is possible that Unit. Scal. approximates the behavior of an SMTO given a few conditions. 

When considering a problem with a shared encoder, the same decoder architecture, the same loss function, task labels with similar distributions, and learning speeds, the gradient norm and the loss value of both tasks are expected to behave in remarkably similar ways. Additionally, if we restrict this assumption to equal gradient norm or equal loss values, it is possible to theoretically show that most SMTOs converge to equal weights (see Appendix~\ref{app:proofs}).

To provide empirical evidence for the above statement, we consider two problems from Multi-MNIST: dual classification and dual reconstruction of the digits. In both cases, the architecture initially described in the previous section was reused with the corresponding changes to the decoders. We follow the same hyperparameter optimization previously described for Multi-MNIST. 

\begin{figure}[htbp]
    \centering
    \begin{subfigure}[b]{0.8\textwidth}
        \centering
        \includegraphics[width=\textwidth]{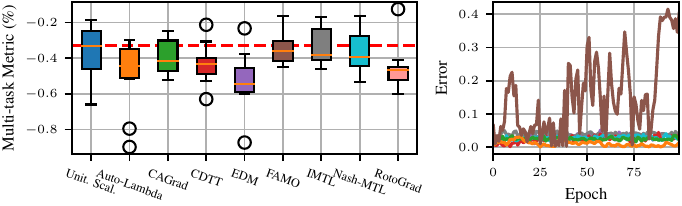}
        \label{fig:CL_CR}
    \end{subfigure}
    
    \vspace{-1.0em}
    
    \begin{subfigure}[b]{0.8\textwidth}
        \centering
        \includegraphics[width=\textwidth]{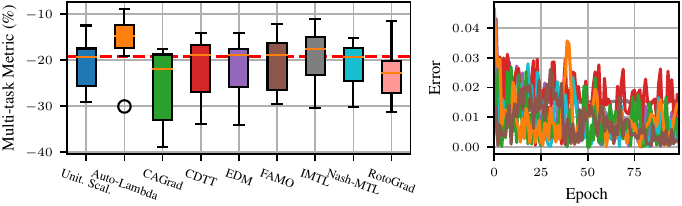}
        \label{fig:RL_RR}
    \end{subfigure}
    
    \caption{Box plot of $\Delta_{\text{mtm}}$ (\textbf{left}) for all SMTOs not evaluated on \cite{kurin2022defense} and mean error to equal weights (\textbf{right}). The top plot corresponds to CL\_CR, and the bottom one represents RL\_RR. All SMTOs had similar performance to Unit. Scal. and the computed weights varied near 0.5, excluding FAMO on the CL\_CR problem, exemplifying how similar equal weights are to the selected SMTOs' solution in this setting.}
    \label{fig:two_similar_tasks}
\end{figure}

The results presented in Figure~\ref{fig:two_similar_tasks} show that all SMTOs performed closely to the Unit. Scal. baseline, even on the dual reconstruction case where $\Delta_{mtm} \approx -20\%$, meaning that there was high interference between tasks. We can also observe that the learned weights from the SMTOs predominantly clustered around 0.5, besides FAMO, which was a notable exception, exhibiting instability on the dual classification problem, likely attributable to the small scale of the loss values.

Consequently, employing equal weights approximates the average behavior of the SMTOs, resulting in comparable performance. This observation suggests that a straightforward averaging of task gradients can serve as a surprisingly effective baseline without significant performance degradation in scenarios with similar task characteristics, such as the dual translation setting examined in the recent critique \cite{xin2022current}.

\subsection{Complexity analysis}
\label{sec:complexity}

Cityscapes \cite{Cordts2016Cityscapes} is a street scene dataset containing high-resolution images 2048x1024 and labels for semantic segmentation, disparity estimation, and instance segmentation. It is commonly adopted for MTL \cite{kendall2017multi, sener2018multi, kurin2022defense, navon2022nashmulti, xin2022current, liu2023famo} under two variations, the two-task problem with 7 class semantic segmentation and disparity estimation, and the three-task problem that uses 19 semantic classes and adds instance segmentation. Some similarities arise when comparing the first version to the Multi-MNIST classification-reconstruction problem. Both problems generate gradients with different norms, and the tasks differ in complexity, yet there is no significant difference between SMTOs and Unit. Scal. with optimized hyperparameters \cite{kurin2022defense, xin2022current}.

\begin{figure}[htbp]
    \centering
    \hfill
    \begin{subfigure}[b]{0.415\textwidth}
        \centering
        \includegraphics[width=\textwidth]{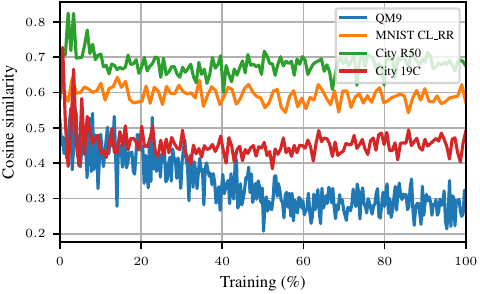}
        \subcaption{Mean cosine similarity.}
        \label{fig:grad_cos_sim}
    \end{subfigure}
    \hfill
    \begin{subfigure}[b]{0.415\textwidth}
        \centering
        \includegraphics[width=\textwidth]{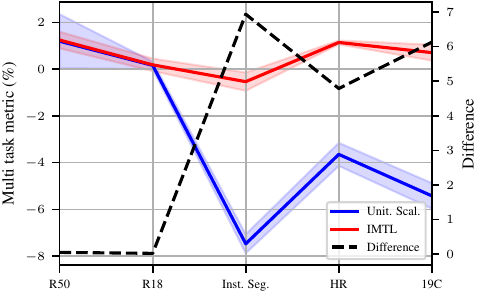}
        \subcaption{IMTL and Unit. Scal. performance.}
        \label{fig:city_complexity}
    \end{subfigure}
    \hfill
    \caption{\textbf{(a)} Mean cosine similarity between each task gradient and the average direction. City R50 exhibits significantly higher similarity than Multi-MNIST, potentially explaining Unit. Scal.'s comparable performance to SMTOs. Conversely, the lower similarity for City 19C and QM9 correlates with Unit. Scal.'s relative underperformance. See Appendix~\ref{app:cos_sim} for the full version with all configurations. \textbf{(b)} Shows the relative performance of Unit. Scal. and IMTL. Adding a third task and increasing the amount of information needed by increasing the number of classes are the relevant steps that dictate the relative performance in this problem.}
\end{figure}

The reason for this unexpected behavior is related to the lower amount of interference between each task compared to the Multi-MNIST problem, as seen in Figure~\ref{fig:grad_cos_sim}. Figure~\ref{fig:grad_cos_sim} shows the mean cosine similarity between each task gradient and the average direction, which can be seen as a measure of interference, as it correctly ranks the problems with respect to the relative performance between Unit. Scal. and SMTOs. Even after significantly reducing the parameter count in step R18, the level of interference remains about the same (see Appendix~\ref{app:cos_sim}), as reflected in the relative performance between Unit. Scal. and IMTL. However, after increasing the complexity of the problem, there is a significant difference between the two methods. 

To vary the complexity in a controlled manner, we vary model size, number of tasks, and information available to learn to identify which characteristic is relevant. We first trained the model on the same initial two-task configuration of \cite{kurin2022defense} and applied small cumulative steps until the most complex configuration. We opted to test only IMTL \cite{liu2021towards} and compare it to Unit. Scal. in the transitional steps as it was one of the best performing SMTOs, and it does not require any specific hyperparameter optimization. All Cityscapes configurations are optimized varying the learning rate at $lr$ $\in$ \{0.005, 0.001, 0.0005, 0.0001\} and the weight decay $\in$ \{$10^{-4}$, $10^{-5}$, 0.0\} with 32 batch size. Each configuration is trained once, and then the best in the validation dataset is trained a total of 3 times. The selected steps are:

\begin{itemize}
    \item \textbf{Reference (R50):} A Resnet-50 \cite{he2016deepresidual} is used as shared encoder and the DeepLab decoder \cite{chen2018encoder_decoder_atrous_conv} is used as task specific decoder. The input resolution is 256x128, and the objective is set to 7 class segmentation and disparity estimation. This is the same configuration of \citet{kurin2022defense} and very similar to the one of \citet{xin2022current};
    \item \textbf{Resnet-18 as decoder (R18):} Opting for a smaller architecture should encourage competition for parameters;
    \item \textbf{Added instance segmentation (Inst. Seg.):} Adding more tasks increases the complexity of the balancing problem, as well as the amount of information available to learn;
    \item \textbf{Increased resolution (HR):} The resolution is increased from 256x128 to 512x256 as an intermediate step to allow better learning of objects from smaller classes;
    \item  \textbf{19 semantic classes (19C):} By increasing the number of learnable classes and keeping the model size fixed, the amount of information needed by the semantic decoder is increased, resulting in more competition for the shared encoder parameters and feature space.
\end{itemize}

The results in Figure~\ref{fig:city_complexity} reveal that the benefit of employing an SMTO only becomes apparent upon the inclusion of the instance segmentation task. Furthermore, a substantial relative performance gain is observed after incorporating all 19 classes, suggesting that the volume of information significantly impacts the efficacy of SMTOs. This observation aligns with prior research on information transfer \cite{Wu2020Understanding}, which highlights the necessity of carefully calibrating the shared model's capacity in multi-task learning. Insufficient or excessive capacity can lead to negative or negligible transfer. Therefore, by increasing the number of tasks and classes, we inherently increase the demands on the shared feature space, potentially increasing inter-task conflicts.

Similar behavior was also noticed in the NLP field \cite{shaham2022interferenceMLT}, where they compared the interference caused by different languages and data size, concluding that a bigger model could mitigate the interference issue, though their evaluation was restricted to Unit. Scal.. On the other hand, we investigate the effect of those parameters on the relative performance of different SMTOs to the baseline.

This exemplifies that SMTOs are not needed on some occasions when the problem is simple enough that the conflicting gradients do not interfere with the final performance. Therefore, the best way to evaluate SMTOs is to use the most complex configuration to measure their capabilities in dealing with significant interference. Accordingly, the other promising SMTOs are trained and evaluated with the same procedures on the last configuration.

Figure~\ref{fig:city_box_plot} shows that SMTOs can indeed achieve better results when compared to Unit. Scal. in more complex balancing problems. This conclusion differs from Unit. Scal. \cite{kurin2022defense} primarily because of the increased complexity. Furthermore, we evaluate some of the newer SMTOs, showing the field's progress in the last years.

\begin{figure}
    \centering
    \includegraphics[width=0.85\linewidth]{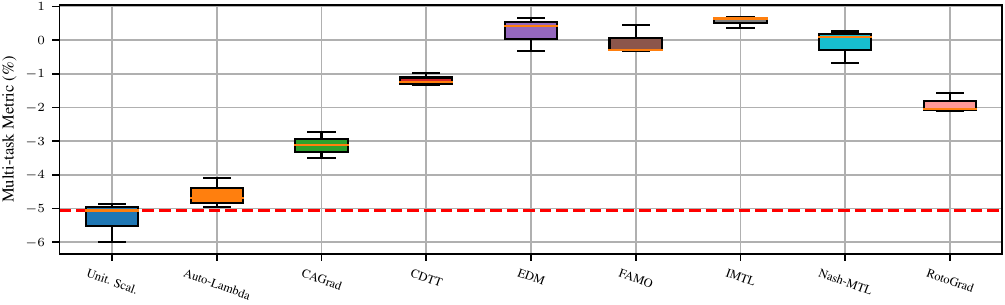}
    \caption{Muiti task metric on Cityscapes dataset. Most SMTOs significantly surpass the Unit. Scal. baseline on the most complex version of the Cityscapes MTL problem.}
    \label{fig:city_box_plot}
\end{figure}

\subsection{QM9}
Finally, the last dataset we use to evaluate SMTOs is the QM9 dataset \cite{wu2018qm9}. The QM9 dataset is a collection of quantum chemistry data for small organic molecules. It contains around 130,000 data points, and it was recently adopted as a problem for the MTL benchmarks \cite{navon2022nashmulti, liu2023famo}.

\begin{wraptable}{r}{0.432\textwidth}
\caption{Multi-task metric on QM9 dataset. Almost all SMTOs significantly surpass Unit. Scal.}
\begin{tabular}{lcc}
\toprule
SMTO & $\Delta_{MTM}$(\%) $\uparrow$ & MR $\downarrow$\\

\midrule
Unit. Scal.  & -135.7 & 4.45 \\
Auto-Lambda & -186.9 & 5.82 \\
CAGrad & -70.13 & 3.82 \\
CDTT & -121.2 & 5.45 \\
EDM & -95.42 & 4.00 \\
FAMO & -56.86 & 2.36 \\
Nash-MTL & \textbf{-53.62} & \textbf{2.09} \\
\bottomrule
\end{tabular}
\label{tab:qm9}
\end{wraptable}

We use the same configuration from \cite{navon2022nashmulti}. The example from PyTorch Geometric \cite{fey2019pytorchGeo} was adapted, and we used the model from \cite{gilmer2017messagepassing}. The model was used to predict 11 properties from the molecules. We use the commonly adopted distribution of 110K molecules for training, 10K for validation, and 10K for testing, with a batch size of 120. The learning rate was optimized in $lr$ $\in$ \{$5\times10^{-3}$, $10^{-3}$, $5\times10^{-4}$\}, and the ground truth values were normalized. We opted not to apply regularization because Unit. Scal. did not show significant improvement, and omitting this step reduced the overall computational cost.

Comparing the results from Table~\ref{tab:qm9} to the previous experiments, it is evident that this is the problem with the most amount of interference between tasks, with all SMTOs obtaining results more than 50\% worse on average compared to the single task problem. This also correlates with the lower cosine similarity shown in Figure~\ref{fig:grad_cos_sim}.

During the experiment, IMTL was too unstable, and most runs crashed due to numerical overflow, as happened on the RL benchmark from Unit. Scal. \cite{kurin2022defense}, which also had a higher number of tasks. We agree with the hypothesis that this was caused by a lack of bounds on the scaling coefficients obtained from the IMTL's optimization problem. RotoGrad also failed to converge, and our hypothesis is that RotoGrad was not able to find a proper rotation that minimizes the gradient conflict probably due to the increased difficulty of the problem. Further analysis can be seen in Appendix~\ref{app:rotograd}. Both were removed from Table~\ref{tab:qm9}.

\subsection{Fixed weights}
\label{sec:fixed_weights}

Though fixed weights are the most efficient and straightforward solution, choosing the correct values for the weights of each task is a challenging problem. \citet{xin2022current} finds the ideal, or close to ideal, weights by performing a grid search. This process may be viable for a more limited scenario, but for a large number of tasks it is not feasible. Very recently, \citet{royer2024scalarization_mt_at_scale} proposed a population-based training to find the ideal weights efficiently, though this method was not evaluated in this work due to time constraints.

One recent critique of SMTOs raises the question of whether optimized scalar weights can compete with SMTOs \cite{xin2022current}. However, as discussed previously, the problems used in the critique might not adequately represent the performance of SMTOs, so we evaluate this claim on our experiments.

Finding well optimized weights for all problems using grid search would be unfeasible, so instead, we opt for a simpler strategy, using the weights output from an SMTO to weight each task loss. More specifically, the weights from CAGrad, EDM and Nash-MTL optimizers are used for the MNIST, Cityscapes and QM9 datasets, respectively. As we are using the Adam optimizer, there is a normalizing effect on the gradients, so the gradients in each decoder are not severely affected by the different norms caused by the loss weighting.

To select the weights, we computed the mean value in each epoch, applied an exponential moving average with $\beta=0.9$, and used the result from the last epoch to define the weights. To extract and train with the fixed weights, we used the optimal hyperparameter configuration from the respective SMTO. However, for the QM9 dataset the fixed weights demonstrated to be unstable so we used the ones from $lr=0.0001$, instead of the optimal $lr=0.001$. The model is trained the same number of times as its SMTO counterpart.

\begin{wrapfigure}{r}{0.45\textwidth}
    \centering
    \includegraphics[width=0.95\linewidth]{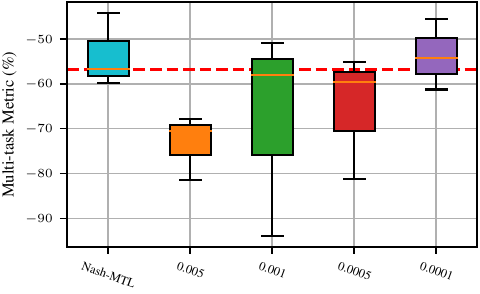}
    \caption{Box plot comparing $\Delta_{\text{mtm}}$ between the original Nash-MTL optimizer and their extracted weights from each learning configuration. Fixed weights can achieve results comparable to those of an SMTO. However, they can be significantly more unstable than a dynamic SMTO, as was the case for the weights from $lr=0.001$.}
    \label{fig:QM9_fw_box_plot}
\end{wrapfigure}

The results obtained from fixed weights, summarized in Table~\ref{tab:fixed_weights}, indicate that it is indeed possible to get competitive performance using fixed weights. However, fixed weights can be more unstable, as shown in Figure~\ref{fig:QM9_fw_box_plot} by the fixed weights extracted from $lr=0.001$, with results ranging from $\approx-50\%$ to $\approx-95\%$. Surprisingly, the optimal learning rate weights were not the best option for the QM9 dataset.

\begin{table}
\centering
\caption{Comparison between Unit. Scal, fixed weights, and the best performing SMTO. Well chosen fixed weights can achieve comparable performance to SMTOs.}
\begin{tabular}{*{4}{c}}
\toprule
\multirow{2}{*}{Dataset} & Unit. Scal. & Fixed Weights & Best SMTO \\
& \multicolumn{3}{c}{$\Delta_{mtm} (\%) \uparrow$} \\
\midrule
MNIST & -8.952 & -3.251 & -2.137 \\
City 19C & -5.312 & 0.145 & 0.570 \\
QM9 & -135.7 & -53.66 & -53.62 \\
\bottomrule
\end{tabular}
\label{tab:fixed_weights}
\end{table}

This raises the question of why fixed weights can achieve comparable performance to SMTOs, as generating dynamic weights is a major benefit of SMTOs. Analyzing the weight value behavior in Figure~\ref{fig:fixed_weights}, it is clear that all weights converge to a fixed value later in training. Additionally, due to the over-parameterized nature of deep learning models \cite{frankle2018the_lottery_ticket} the limitations typically associated with under-parameterized settings, such as the claim that scalarization is generally incapable of tracing out the Pareto front, do not necessarily apply \cite{hu2023revisiting}. We believe that further exploring the interaction between over-parameterization and multi-task learning could be a promising direction for future work.

This shows that using fixed weights is a valid alternative for complex MTL problems, corroborating the findings of \citet{xin2022current}. However, finding those weights is highly costly, and with the advancement of loss-based SMTOs such as FAMO \cite{liu2023famo} and GO4Align \cite{shen2024go4align}, it is no longer a desirable choice.

\begin{figure*}[t] % * makes it span across both columns
    \centering
    \hfill
    \begin{subfigure}{0.3\linewidth}
        \centering
        \includegraphics[width=\linewidth]{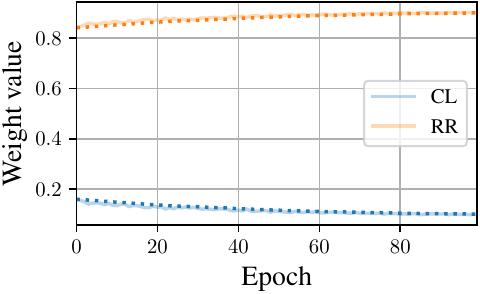}
        % \caption{MNIST}
        \label{fig:top_left}
    \end{subfigure}
    \hfill
    \begin{subfigure}{0.3\linewidth}
        \centering
        \includegraphics[width=\linewidth]{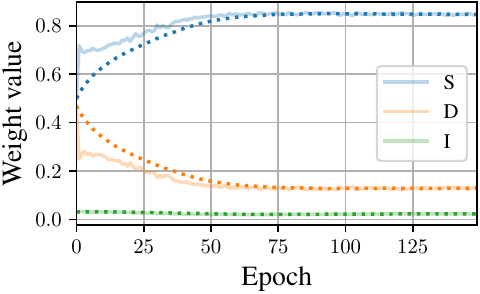}
        % \caption{Cityscapes}
        \label{fig:top_right}
    \end{subfigure}
    \hfill
    \vspace{-1em} % space between top and bottom rows
    \begin{subfigure}{0.98\linewidth}
        \centering
        \includegraphics[width=\linewidth]{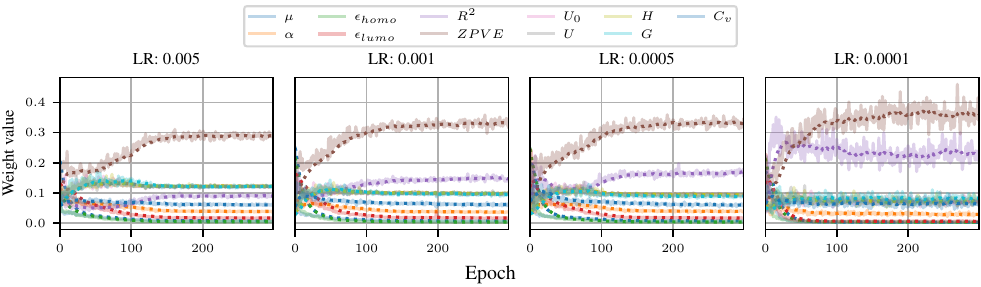}
        % \caption{QM9 dataset.}
        \label{fig:bottom}
    \end{subfigure}
    % \vspace{-1em} % space between top and bottom rows
    
    \caption{Normalized weights from SMTOs. The shaded regions represent the mean value at each epoch, and the dotted lines show the results of the exponential moving average. The \textbf{top left}, \textbf{top right}, and \textbf{bottom} plots correspond to the MNIST, Cityscapes, and QM9 datasets, respectively. All datasets exhibit similar behavior, with the weights gradually converging to a stable value on average.}
    \label{fig:fixed_weights}
    % \vskip -0.1in
\end{figure*}

\section{Limitations}
\label{sec:limitations}
In this work, we extensively tested multiple SMTOs on multiple configurations and datasets. However, it is important to note that we only focused on the supervised setting and did not include problems from NLP or reinforcement learning, fields in which MTL is commonly evaluated \cite{kurin2022defense, xin2022current}. Also, our initial selection process involved filtering certain SMTOs based on their norm balancing performance on the MNIST classification-reconstruction problem. Despite this strong preliminary evidence, a more exhaustive analysis of the excluded methods may reveal unknown advantages. Furthermore, some SMTOs were reported to provide better results when combining them with other SMTOs \cite{javaloy2021rotograd, liu2022autolambda}, but because of time constraints, we limited the experiments to only the vanilla version. Finally, we only considered the MTL problem, where all tasks are equally weighted, and did not consider scenarios designed for auxiliary tasks.

\section{Conclusion}

This work aimed to address recent critiques of SMTOs by conducting a comprehensive empirical analysis under the context of supervised learning. First, we examined why Unit. Scal. can replicate the performance of SMTOs in two-task scenarios, particularly when the tasks are similar. Our findings suggest that SMTOs converge to equal weights, so Unit. Scal. approximates their behavior on average. Furthermore, we explored the conditions under which Unit. Scal. matches SMTOs, attributing its success to factors such as the degree of task interference, the number of tasks, and the amount of information to be learned.

We also investigated the relative performance of fixed weights compared to SMTOs. By deriving fixed weights from SMTOs, we demonstrated that fixed weights can perform competitively in complex scenarios. However, using fixed weights is less appealing due to the computational overhead of identifying optimal values and their comparable performance to recent loss-based SMTOs.

In conclusion, this study provides a robust evaluation of current SMTOs, underscoring their significance in addressing complex multi-task problems. Additionally, we aim to establish points to consider when evaluating SMTOs and to guide the development of future methods and applications.

\section*{Acknowledgments}

This work was partially funded by CNPq under grants 465755/2014-3, 309532/2023-0, and 130961/2024-8, and CAPES under finance code 001.

\bibliographystyle{unsrtnat}
\bibliography{neurips_2025}

\appendix

\section{Technical Appendices and Supplementary Material}

\subsection{Experimental Setup}
\label{app:hardware}
All experiments were implemented using PyTorch~\cite{NEURIPS2019_bdbca288_pytorch} and executed on a machine equipped with an AMD Ryzen 9 5950X CPU and two NVIDIA RTX 3090 GPUs (24\,GB each). However, experiments were run independently on each GPU. The random seed was set based on the corresponding run index.

\subsection{Additional results}

\subsubsection{Cosine similarity}
\label{app:cos_sim}

Figure~\ref{fig:cos_sim_all} demonstrates the strong correlation between the mean cosine similarity of task gradients and the relative performance of Unit. Scal. and SMTOs. Notably, the tasks where Unit. Scal. achieved comparable results to SMTOs (MNIST CL\_CR, MNIST RL\_RR, City R50, and City R18) exhibit high mean cosine similarities ($\approx$0.7) and form a distinct cluster. Conversely, the more challenging Cityscapes tasks (City Inst. Seg., City HR, City 19C) show a lower, yet consistent, similarity ($\approx$0.45). Finally, the most difficult task from the QM9 dataset, is characterized by the lowest mean cosine similarity ($\approx$0.3).

\begin{figure}
    \centering
    \includegraphics[width=\linewidth]{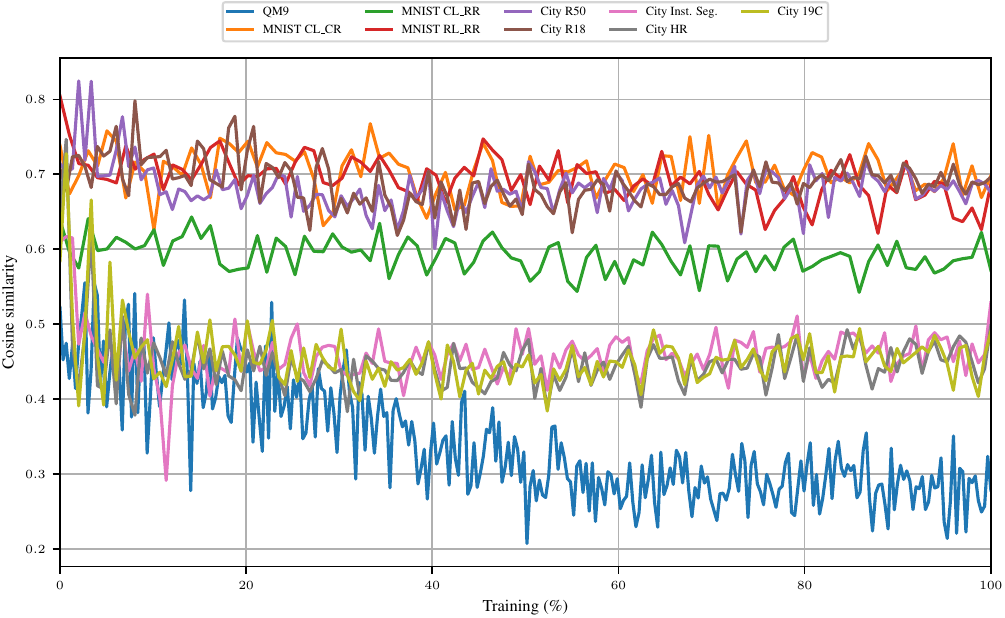}
    \caption{Mean cosine similarity between each task gradient and the average direction for all problems. It was able to correctly rank all the different problems in terms of relative performance between Unit. Scal. and SMTOs. (i.e. higher cosine similarity relates to lower difference between Unit. Scal. and SMTOs.)}
    \label{fig:cos_sim_all}
\end{figure}

\subsubsection{Statistics}
\label{app:add_statistics}
Similarly to \cite{javaloy2021rotograd}, we provide additional statistics in Tables~\ref{tab:app_mnist},~\ref{tab:app_city}, and~\ref{tab:app_qm9_stat} and the missing visual result in Figure~\ref{fig:qm9_box_plot}. Most notably, we see a clear preference for the classification task on the MNIST dataset from the underperforming SMTOs, as well as a lower minimum value and higher standard deviation in Table~\ref{tab:app_qm9_stat} for Unit. Scal. when compared to the top performing SMTOs, FAMO and Nash-MTL.

\begin{table}
\centering
\caption{Mean $\Delta_{mtm} (\%) \uparrow$ of different SMTOs on the MNIST dataset.}
\begin{tabular}{lccc}
\toprule
\multirow{2}{*}{SMTO} & \multirow{2}{*}{CL} & \multirow{2}{*}{RR} & \multirow{2}{*}{avg} \\
 & & & \\
\midrule
Unit. Scal. & -1.043 & -16.86 & -8.952 \\
Auto-Lambda & -1.232 & -7.186 & -4.209 \\
CAGrad & -1.685 & -2.590 & \textbf{-2.137} \\
CDTT & -1.546 & -3.596 & -2.571 \\
EDM & -1.613 & -2.993 & -2.303 \\
FAMO & -4.587 & \textbf{-0.745} & -2.666 \\
GradDrop & -1.076 & -22.04 & -11.56 \\
IMTL & -2.513 & -1.975 & -2.244 \\
MGDA-UB & -1.448 & -11.26 & -6.352 \\
Nash-MTL & -1.695 & -4.430 & -3.063 \\
PCGrad & -1.039 & -15.36 & -8.201 \\
RLW-Dirichlet & -1.145 & -19.48 & -10.31 \\
RLW-Normal & -1.100 & -18.53 & -9.817 \\
RotoGrad & -2.705 & -1.572 & -2.139 \\
SI & \textbf{-0.962} & -27.47 & -14.21 \\
UW & -1.109 & -26.30 & -13.70 \\
\bottomrule
\end{tabular}
\label{tab:app_mnist}
\end{table}

\begin{table}
\centering
\caption{Mean $\Delta_{mtm} (\%) \uparrow$ of different SMTOs on the Cityscapes dataset.}
\begin{tabular}{lcccccc}
\toprule
\multirow{2}{*}{SMTO} & \multicolumn{2}{c}{Segmentation} & \multicolumn{2}{c}{Disparity} & Instance Seg. & \multirow{2}{*}{avg} \\
\cmidrule(lr){2-3}\cmidrule(lr){4-5}\cmidrule(lr){6-6}
 & Acc & mIoU & L1 abs & L1 rel & L1 abs & \\
\midrule
Unit. Scal.  & -2.120 & -16.78 & -8.542 & -11.41 & 3.492 & -5.312 \\
Auto-Lambda  & -1.922 & -15.02 & -8.546 & -9.317 & \textbf{3.657} & -4.582 \\
CAGrad  & -0.977 & -7.647 & -6.709 & -10.06 & 3.325 & -3.124 \\
CDTT  & -0.948 & -8.954 & 2.043 & -0.224 & 0.472 & -1.190 \\
EDM  & -0.225 & -3.174 & 0.990 & -1.051 & 2.468 & 0.246 \\
FAMO  & -0.274 & -3.268 & 0.755 & -2.212 & 2.339 & -0.054 \\
IMTL & \textbf{-0.224} & \textbf{-2.759} & \textbf{3.017} & \textbf{0.993} & 1.195 & \textbf{0.570} \\
Nash-MTL & -0.290 & -3.275 & 0.297 & -0.227 & 1.409 & -0.113 \\
RotoGrad & -0.453 & -4.484 & 0.303 & -3.493 & -1.659 & -1.908\\
\bottomrule
\end{tabular}
\label{tab:app_city}
\end{table}

\begin{table}
\centering
\caption{Performance of different SMTOs on the QM9 dataset based on $\Delta_{mtm} (\%)$ statistics.}
\begin{tabular}{lccccc}
\toprule
SMTO &  min  $\uparrow$ & max $\uparrow$ & med $\uparrow$ & std $\downarrow$ & avg $\uparrow$ \\
\midrule
Unit. Scal. & -849.3 & -8.280 & -61.62 & 229.4 & -135.7 \\
Auto-Lambda & -974.5 & -18.32 & -115.4 & 255.5 & -186.9 \\
CAGrad & -408.0 & -5.451 & -32.08 & 108.7 & -70.13 \\
CDTT & -652.9 & -30.90 & -70.34 & 169.7 & -121.2 \\
EDM & -526.0 & -24.14 & -56.26 & 137.0 & -95.42 \\
FAMO & \textbf{-351.5} & 6.590 & \textbf{-16.95} & \textbf{97.78} & -56.86 \\
Nash-MTL & -390.2 & \textbf{15.10} & -21.41 & 108.8 & \textbf{-53.62} \\
\bottomrule
\end{tabular}
\label{tab:app_qm9_stat}
\end{table}

\begin{figure}[ht]
    \centering
    \includegraphics[width=1.0\linewidth]{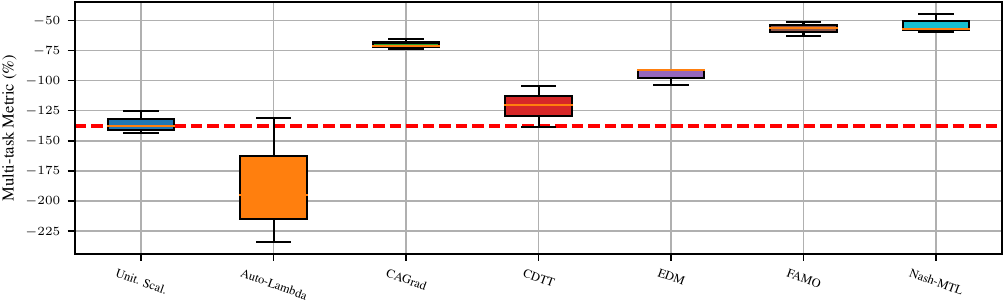}
    \caption{Box plot comparing $\Delta_{mtm}$ between Unit. Scal. and SMTOs on the QM9 dataset.}
    \label{fig:qm9_box_plot}
\end{figure}

\subsection{Rotograd's instability on the QM9 dataset}
\label{app:rotograd}

The RotoGrad optimizer aims to mitigate conflicts in gradient directions by rotating the feature space so that the gradient of each task's gradient aligns with the direction formed by the equally weighted sum of the normalized gradients: $v = \frac{1}{N}\sum_i^N u_{sh,i}$, where $u_{sh,i} = \frac{g_{sh,i}}{\|g_{sh,i}\|}$ represents the normalized gradient of task $i$ for the shared parameters.

In the specific case of the QM9 experiment, RotoGrad failed to converge. We hypothesize that the relatively small final shared feature space, combined with the high number of tasks, significantly increased the difficulty of the balancing problem. As shown in Figure~\ref{fig:rotation_loss}, the rotation loss $\mathcal{L}_{rot}$ is orders of magnitude higher compared to the MNIST and Cityscapes problems. Since $\mathcal{L}_{rot}$ does not directly enforce a reduction in task loss, this instability may have negatively affected the model optimization.

For this specific experiment, we used $lr=0.001$ as it was the optimal learning rate for most SMTOs, and $lrr=0.1$ since it was the smallest value used on the grid search, so ideally it should be the most stable.

\begin{figure}[ht]
    \centering
    \includegraphics[width=0.95\linewidth]{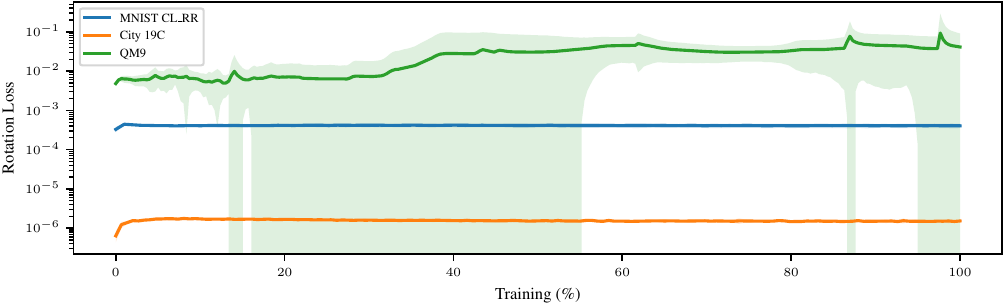}
    \caption{Mean rotation loss $\left(\mathcal{L}_{rot}\right)$ from RotoGrad. The $\mathcal{L}_{rot}$ is at least 10$\times$ higher than the other cases, which could explain the training instability.}
    \label{fig:rotation_loss}
\end{figure}

\subsection{Proofs}
\label{app:proofs}

\begin{assumption}
\label{assp:two-task-scenario-grad-norm}
Consider a model $f_{\theta}$ composed of a shared encoder $f_{\theta_{sh}}$ and task-specific modules $f_{\theta_t}$ for each task $t \in \{1, \dots, N\}$. In the case of two tasks ($N=2$), we assume the following symmetry conditions: \textbf{(i)} both tasks use the same loss function, \textbf{(ii)} the task-specific modules have the same architecture, \textbf{(iii)} the tasks are of similar difficulty, and \textbf{(iv)} the tasks have labels with similar distributions. Under these conditions, we assume that the gradients with respect to the shared parameters have equal norm, i.e., $\|g_1\| = \|g_2\|$.
\end{assumption}

\begin{assumption}
\label{assp:two-task-scenario-loss-value}
Under the same conditions as Assumption~\ref{assp:two-task-scenario-grad-norm}, assume that the loss values of each task are equal, i.e., $\ell_1 = \ell_2$.
\end{assumption}

Although Assumptions~\ref{assp:two-task-scenario-grad-norm},~\ref{assp:two-task-scenario-loss-value} impose a strong symmetry requirement, our experiments on MNIST revealed that in both the dual classification and dual reconstruction tasks, the resulting weights from all SMTOs were consistently close to 0.5, with the sole exception of FAMO in the dual classification scenario.

\begin{proposition}
\label{prop:cagrad-equal-weights}
Under Assumption~\ref{assp:two-task-scenario-grad-norm}, the optimal weights computed by CAGrad correspond to equal weighting.
\end{proposition}

\begin{proof}
We consider the CAGrad optimization for two tasks:
%\vspace{-0.45cm}
\[
w^* = \mathop{\arg\min}_{w \in \mathcal{W}} \; g_w^\top g_0 + \sqrt{\phi} \, \|g_w\|,
\]
where:
\begin{align*}
g_w &= w_1 g_1 + w_2 g_2, \\
g_0 &= \frac{1}{2}(g_1 + g_2), \\
\phi &= c^2 \|g_0\|^2, \quad c \in [0,1), \\
\mathcal{W} &= \{w \in \mathbb{R}^2 \mid w_1 + w_2 = 1,\; w_1, w_2 \geq 0\}.
\end{align*}
We also define:
\begin{align*}
    A &\coloneqq \|g_1\| = \|g_2\|, \\
    p &\coloneqq g_1^\top g_2.
\end{align*}
We begin by expanding the inner product:
\begin{align*}
g_w^\top g_0 &= \left(w_1 g_1 + w_2 g_2\right)^\top \frac{1}{2}(g_1 + g_2) \\
&= \frac{1}{2} \left[w_1 (g_1^\top g_1 + g_1^\top g_2) + w_2 (g_2^\top g_1 + g_2^\top g_2)\right] \\
&= \frac{1}{2} \left[w_1 (A^2 + p) + w_2 (p + A^2)\right] \\
&= \frac{1}{2}(A^2 + p)(w_1 + w_2) = \frac{1}{2}(A^2 + p).
\end{align*}
This term is independent of $w$, as well as $\sqrt{\phi}$, so the optimization reduces to:
\[
\mathop{\arg\min}_{w \in \mathcal{W}} \|w_1 g_1 + w_2 g_2\|.
\]
Let $w^*$ be the optimal solution, which is known to be equal weights if $\|g_1\| = \|g_2\|$. The CAGrad update direction is then given by:
\[
d^* = g_0 + \frac{\sqrt{\phi}}{\|g_{w^*}\|} \, g_{w^*}.
\]
If $w^* = (\tfrac{1}{2}, \tfrac{1}{2})$, then $g_{w^*} = \frac{1}{2}(g_1 + g_2) = g_0$, and:
\[
d^* = g_0 + \frac{c \|g_0\|}{\|g_0\|} \, g_0 = (1 + c) \, g_0 = \sum_{i=1}^{2} g_i \left(\frac{1 + c}{2}\right).
\]
Thus, $d^*$ is an equal-weight combination of the gradients.
\end{proof}

\begin{proposition}
Under Assumption~\ref{assp:two-task-scenario-grad-norm}, the optimal weights computed by Nash-MTL \cite{navon2022nashmulti} are equal weights in the two-task scenario.
\end{proposition}

\begin{proof}
The weights are computed by approximating the solution of:

\[
    G^\top G\alpha = \frac{1}{\alpha},
\]
where $G$ is the matrix whose columns are the shared gradients and $\alpha_i \; > 0, \; i \in \{1, ..., N\}$ the tasks' weights. In a two-task scenario, the problem can be simplified to:
\[
\left\{
\begin{aligned}
\alpha_1 g_1^\top g_1 + \alpha_2g_2^\top g_1 &= \frac{1}{\alpha_1} \\
\alpha_2 g_2^\top g_2 + \alpha_1g_1^\top g_2 &= \frac{1}{\alpha_2}
\end{aligned}
\right.
\]

Defining $A \coloneqq \|g_1\| = \|g_2\|$ and $p \coloneqq \frac{g_1^\top g_2}{A^2}$:
\[
\left\{
\begin{aligned}
\alpha_1 A^2 + \alpha_2A^2p &= \frac{1}{\alpha_1} \\
\alpha_2 A^2 + \alpha_1A^2p &= \frac{1}{\alpha_2}
\end{aligned}
\right.
\]

Rearrange the equations:
\[
\left\{
\begin{aligned}
\alpha_1^2 A^2 + \alpha_1\alpha_2A^2p &= 1 \\
\alpha_2^2 A^2 + \alpha_1\alpha_2 A^2p &= 1
\end{aligned}
\right.
\]

Notice that by symmetry of the two equations, the solution should satisfy $\alpha_1^2=\alpha_2^2$. Since both weights are positive, this immediately implies:
\[
\alpha_1 = \alpha_2.
\]
\end{proof}

\begin{proposition}
Under Assumption~\ref{assp:two-task-scenario-grad-norm}, the optimal weights computed by IMTL-G \cite{liu2021towards} are equal weights in the two-task scenario.
\end{proposition}

\begin{proof}
The weights are obtained by solving:
\[
    \alpha = g_1 U^\top (DU^\top)^{-1},
\]
where $\alpha = \{\alpha_2, ..., \alpha_N\}$ is constrained to $\sum_i \alpha_i = 1$, $U^\top = [u^\top_1 - u^\top_2, ..., u^\top_1 - u^\top_N]$, with $u_i = g_i/\|g_i\|$, and $D^\top = [g_1^\top - g_2^\top, ..., g_1^\top - g_N^\top]$. Define $A \coloneqq \|g_1\| = \|g_2\|$ and $p = g_1^\top g_2$. In the two-task scenario, the problem reduces to:
\begin{align*}
    \alpha_2 &= g_1 \left(\frac{g_1^\top - g_2^\top }{A}\right)\left((g_1-g_2) \frac{(g_1^\top-g_2^\top)}{A}\right) ^ {-1} \\
    \alpha_2 &= g_1 \left(\frac{g_1^\top - g_2^\top }{A}\right) \left( \frac{\|g_1 - g_2\|^2}{A} \right) ^ {-1} \\
    \alpha_2 &= g_1 \left( \frac{g_1^\top - g_2^\top}{\|g_1 - g_2\|^2}\right) = \frac{A^2 - p}{2(A^2 - p)} = \frac{1}{2}.
\end{align*}

Since $\alpha_1 + \alpha_2 = 1$, this completes the proof.

\end{proof}

\begin{proposition}
\label{prop:edm}
Under assumption~\ref{assp:two-task-scenario-grad-norm}, the optimal weights computed by EDM \cite{katrutsa2020follow} are equal weights in the two-task scenario.
\end{proposition}

\begin{proof}
EDM's original paper \cite{katrutsa2020follow} already showed that for the two task scenario the optimal direction is computed as follows:
\[
d^* = \left( \frac{1}{\|g_1\|} + \frac{1}{\|g_2\|} \right)^{-1} \left( \frac{g_1}{\|g_1\|} + \frac{g_2}{\|g_2\|} \right).
\]
Define $A \coloneqq \|g_1\| = \|g_2\|$:
\[
d^* = \left( \frac{1}{A} + \frac{1}{A} \right)^{-1} \left( \frac{g_1}{A} + \frac{g_2}{A} \right) = \frac{1}{2} (g_1 + g_2).  \\
\]

\end{proof}

\begin{proposition}
Under Assumptions~\ref{assp:two-task-scenario-grad-norm} and~\ref{assp:two-task-scenario-loss-value}, the optimal weights computed by CDTT \cite{nakamura2022leveraging} are equal weights.
\end{proposition}

\begin{proof}
CDTT applies a tension vector to the optimal direction $d^*$ computed by EDM, which is equal weights in this case (see Proposition~\ref{prop:edm}). The final vector is computed by:
\[
\begin{aligned}
    \zeta_i &= \frac{1}{N} \sum_{i=1}^N \|g_i(t-i)\| \\
    \delta_i &= \frac{\|\zeta_i(t)\|}{\|\zeta_i(t-1)\|} + \log_{10}(\ell_i)\\
    c_i &= \frac{\alpha}{1 + e^{(-\delta_ie + e)}} + 1 - \alpha\\
    d_n^* &= d^* + \sum^N_{i=1} c_i\left(\frac{g_i - d^*}{\|g_i - d^*\|}\right),
\end{aligned}
\]
where $t$ refers to the iteration index and $\alpha \in [0, 1]$ is a constant to regulate the tension factor's sensitivity.

It follows directly from Assumptions~\ref{assp:two-task-scenario-grad-norm},~\ref{assp:two-task-scenario-loss-value} and the component's equations that $c = c_1 = c_2$. Analyzing the $\|g_i - d^*\|$ separately:
\[
\|g_i - d^*\|^2 = \|g_i\|^2 + \|d^*\|^2 - 2 \langle g_i, d^*\rangle.
\]

The direction computed by EDM $d^*$ is guaranteed to have the same angle with each of the shared gradients \cite{katrutsa2020follow}. Therefore, if $\|g_1\| =\|g_2\|, \;\text{then} \; \langle g_1, d^*\rangle = \langle g_2, d^*\rangle$. Then we can consider $B \coloneqq \|g_1 - d^*\| = \|g_2 - d^*\|$ and remove it together with $c$ from the sum:
\[
\begin{aligned}
d_n^* &= d^* + \frac{c}{B}\sum^2_{i=1} \left(g_i - d^*\right) \\
d_n^* &= d^* + \frac{c}{B}(g_1 + g_2 - 2d^*) \\
d_n^* &= d^* = \frac{1}{2}(g_1 + g_2).
\end{aligned}
\]

\end{proof}

\begin{proposition} Under Assumption~\ref{assp:two-task-scenario-loss-value} (extended for N tasks), if the logits \(\xi_t = (\xi_{1,t}, \dots, \xi_{N,t})\) are initialized equally, then the FAMO algorithm maintains equal task weights:
\[
w_t = w_{i,t} = c_t \frac{z_{i, t}}{\ell_{i, t}},
\]
\textnormal{where} $c_t = \left(\sum_{i=1}^N \dfrac{z_{i, t}}{\ell_{i, t}}\right)$, \textnormal{for all} $i = 1,\dots,N$  \textnormal{and for all} $t$.
%\[
%\begin{aligned}
%w_t &= w_{i,t} = c_t \frac{z_{i, t}}{\ell_{i, t}}, \\ \quad \textnormal{where} \quad c_t &= \left(\sum_{i=1}^N \frac{z_{i, t}}{\ell_{i, t}}\right), \\ \quad \textnormal{for all } i &= 1,\dots,N \textnormal{ and for all } t.
%\end{aligned}
%\]
\end{proposition}

\begin{proof}
We prove the result by induction on the iteration \(t\).

\textbf{Base Case (\(t=0\)):}  
Assume that $\xi_{i,0} = c$ for all $i$, for some constant $c$.
%Assume that
%\[
%\xi_{i,0} = c \quad \text{for all } i, \text{for some constant } c
%\]
%%for some constant \(c\). 
Then the softmax yields
\[
z_{i,0} = \frac{\exp(\xi_{i,0})}{\sum_{j=1}^{N} \exp(\xi_{j,0})} = \frac{\exp(c)}{N\exp(c)} = \frac{1}{N}.
\]

The model parameters are updated as follows:
%\[
%\begin{aligned}
%\theta_{1} &= \theta_0 - \alpha \sum_{i=1}^N \left(c_0\frac{z_{i,0}}{\ell_{i,0}}\right) \nabla \ell_{i,0}, \\ \text{where} \quad c_0 &= \left(\sum_{i=1}^N \frac{z_{i,0}}{\ell_{i,0}}\right)^{-1}.
%\end{aligned}
%\]
\[
\theta_{1} = \theta_0 - \alpha \sum_{i=1}^N \left(c_0\frac{z_{i,0}}{\ell_{i,0}}\right) \nabla \ell_{i,0}, 
\]
where $c_0 = \left(\sum_{i=1}^N \dfrac{z_{i,0}}{\ell_{i,0}}\right)^{-1}$.

Under the assumption \(\ell_{i,t} = \ell_t\) for all \(i\), each term \(\frac{z_{i,t}}{\ell_{i,t}}\) is identical across tasks. Therefore, this update does not introduce any asymmetry among tasks. Thus, the base case holds.

\medskip

\textbf{Inductive Step:}  
Suppose that at iteration \(t\) we have
\[
\xi_t = (c_t, c_t, \dots, c_t),
\]
which implies
\[
z_{i,t} = \frac{\exp(c_t)}{N \exp(c_t)} = \frac{1}{N} \quad \text{for all } i.
\]

Next, the logits are updated via
\[
\xi_{t+1} = \xi_t - \beta \left( \delta_t + \gamma \, \xi_t \right),
\]
with
\[
\delta_t = 
\begin{bmatrix}
\nabla^\top z_{1,t}(\xi_t) \\
\vdots \\
\nabla^\top z_{n,t}(\xi_t)
\end{bmatrix}
\begin{bmatrix}
\log \ell_{1,t} - \log \ell_{1,t+1} \\
\vdots \\
\log \ell_{n,t} - \log \ell_{n,t+1}
\end{bmatrix}.
\]
Since the losses satisfy \(\ell_{i,t} = \ell_t\) for every \(i\) (and similarly for \(t+1\) by the symmetry of the update), the differences \(\log \ell_{i,t} - \log \ell_{i,t+1}\) are the same for all tasks. 

Moreover, the softmax function is given by
\[
z_{i,t} = \frac{\exp(\xi_{i,t})}{\sum_{j=1}^N \exp(\xi_{j,t})}.
\]
The derivative of \(z_{i,t}\) with respect to \(\xi_j\) is
\[
\frac{\partial z_i}{\partial \xi_j} = z_i \left(\delta_{ij} - z_j\right).
\]
where $\delta_{ij}$ is the Kronecker delta, defined as $\delta_{ij} = 1$ if $i=j$ and 0 otherwise. At the symmetric point \(\xi_t = (c_t, \dots, c_t)\), we have \(z_i = \frac{1}{N}\) for all \(i\), hence
\[
\frac{\partial z_i}{\partial \xi_j} = \frac{1}{N}\left(\delta_{ij} - \frac{1}{N}\right).
\]
This shows that the Jacobian of the softmax is the same for every coordinate, ensuring that the scalar products involved in the definition of \(\delta_t\) yield the same value for every task. Let us denote this common value by \(D\), i.e.,
\[
[\delta_t]_i = D \quad \text{for all } i.
\]
Since the regularization term \(\gamma \, \xi_t\) is also applied uniformly, the update for each coordinate becomes
\[
\xi_{i,t+1} = c_t - \beta (D + \gamma c_t) \quad \text{for all } i.
\]
Thus, \(\xi_{t+1} = (c_{t+1}, c_{t+1}, \ldots, c_{t+1})\) remains symmetric. Evidently, the induction holds for the specific case $N=2$.

\end{proof}

\subsection{Dataset Usage and Licensing}
\label{app:dataset_license}
\begin{itemize}
    \item \textbf{MNIST}: Creative Commons Attribution-Share Alike 3.0
    \item \textbf{Cityscapes}: Custom, allows for academic and non-commercial use, \url{https://www.cityscapes-dataset.com/license/}
    \item \textbf{QM9}: We were not able to find a license for this dataset. However the GitHub repository that manages the dataset is under an MIT license \url{https://github.com/pyg-team/pytorch_geometric}.
\end{itemize}

\subsection{Hyperparameters}
\label{app:hyper}

Tables~\ref{tab:mnist_CL_CR_hyper},~\ref{tab:mnist_CL_RR_hyper},~\ref{tab:mnist_RL_RR_hyper},~\ref{tab:city_hyper}, and~\ref{tab:QM9_hyper} summarize the best hyperparameters for each SMTO in each scenario. Also, it shows the best SMTO specific parameters in bold. In the MNIST configuration, these parameters were searched in a broader range to better understand their impact on optimization and limit the search space in the more complex dataset. To determine which specific hyperparameters of each SMTO to optimize, we selected those explicitly optimized in the respective articles. Consequently, we kept Auto-Lambda's initial weight, FAMO's learning rate for task logits, Nash-MTL's optimizer iterations, and GradDrop's parameters fixed throughout the experiments. Additionally, for the Cityscapes dataset, RotoGrad's feature size was set to the maximum possible value or capped at 512, the same used for the ResNet18 model trained on the original Rotograd paper.

For the hyperparameters that were optimized without a defined range, such as Auto-Lambda's and RotoGrad's auxiliary learning rate, we opted to define it as a scale of the learning rate ($aux\_lr = lr*lr\_scale$). 

For the specific case of the CDTT optimizer in the QM9 dataset, we opted to use the same optimal learning rate from EDM and only optimized specific hyperparameters due to time constraints. We justify this choice because CDTT was proposed as a direct improvement from EDM. 

\begin{table}[ht]
    \small
    \centering
    \caption{Hyperparameters for MNIST - CL \& CR.}
    \begin{tabular}{cclcccc}
        \toprule
        \multirow{2}{*}{Dataset} & \multirow{2}{*}{\centering Tasks} & \multirow{2}{*}{\centering SMTO} & \multicolumn{3}{c}{Hyperparameters} & \multirow{2}{*}{\centering SMTO specific parameters} \\
        \cmidrule{4-6}
        & & & LR & P & WD & \\
        \midrule
        MNIST & \makecell{CL \\ CR} & Single Task & \makecell{0.0025 \\ 0.005} & \makecell{0.5 \\ 0.3} & \makecell{0 \\ 0} & \\ 
        \midrule
        & & Unit. Scal. & 0.0025 & 0.3 & 0 & \\
        & & Auto-Lambda & 0.005 & 0.3 & 0 & \makecell{auxiliary learning rate scale: \\ \{1000, 100, 10 1.0 \textbf{0.1} 0.01 0.001\} \\ initial weight: \textbf{0.1}}\\
        & & CAGrad & 0.0025 & 0.3 & 0 & $c$: \{\textbf{0.2}, 0.5, 0.8\} \\
        & & CDTT & 0.005 & 0.2 & 0 & $\alpha$: \{0.2, 0.4, \textbf{0.6}, 0.8, 1.0\} \\
        & & EDM & 0.005 & 0.2 & 0 & \\
        & & FAMO & 0.0025 & 0.2 & 0 & \makecell{$\gamma$: \{0.01, \textbf{0.001}, 0.0001\} \\ learning rate of the task logits: \textbf{0.025}} \\
        & & IMTL & 0.0025 & 0.3 & 0 & \\
        & & Nash-MTL & 0.0025 & 0.3 & 0 & optimizer iterations: \textbf{20}\\
        & & RotoGrad & 0.005 & 0.3 & 0 & \makecell{auxiliary learning rate scale: \\ \{5.0 1.0 0.5 \textbf{0.1}\} \\ feature size: \textbf{50}}\\
        \bottomrule
    \end{tabular}
    \label{tab:mnist_CL_CR_hyper}
\end{table}

\begin{table}[ht]
    \small
    \centering
    \caption{Hyperparameters for MNIST - CL \& RR.}
    \begin{tabular}{cclcccc}
        \toprule
        \multirow{2}{*}{\centering Dataset} & \multirow{2}{*}{\centering Tasks} & \multirow{2}{*}{\centering SMTO} & \multicolumn{3}{c}{Hyperparameters} & \multirow{2}{*}{\centering SMTO specific parameters} \\
        \cmidrule{4-6}
        & & & LR & P & WD & \\
        \midrule
        MNIST & \makecell{CL \\ RR} & Single Task & \makecell{0.0025 \\ 0.0025} & \makecell{0.5 \\ 0.0} & \makecell{0 \\ 0} & \\
        \midrule
        & & Unit. Scal. & 0.001 & 0 & 0 & \\
        & & Auto-Lambda & 0.001 & 0 & 0 & \makecell{auxiliary learning rate scale: \\ \{1000 \textbf{100} 10 1 0.1 0.01 0.001\} \\ initial weight: \textbf{0.1}} \\
        & & CAGrad & 0.001 & 0 & 0 & $c$: \{0.2, 0.5, \textbf{0.8}\} \\
        & & CDTT & 0.0025 & 0 & 0 & $\alpha$: \{\textbf{0.2}, 0.4, 0.6, 0.8, 1.0\} \\
        & & EDM & 0.001 & 0 & 0 & \\
        & & FAMO & 0.0025 & 0 & 0 & \makecell{$\gamma$: \{0.01, 0.001, \textbf{0.0001}\} \\ learning rate of the task logits: \textbf{0.025}} \\
        & & GradDrop & 0.001 & 0 & 0 & \makecell{k: \textbf{1.0} \\ p: \textbf{0.5}}\\
        & & IMTL & 0.001 & 0 & 0 & \\
        & & MGDA-UB & 0.001 & 0 & 0 & \\
        & & Nash-MTL & 0.001 & 0 & 0 & optimizer iterations: \textbf{20}\\
        & & PCGrad & 0.001 & 0 & 0 & \\
        & & RLW-Dirichlet & 0.001 & 0 & 0 & \\
        & & RLW-Normal & 0.001 & 0 & 0 & \\
        & & RotoGrad & 0.001 & 0 & 0 & \makecell{auxiliary learning rate scale: \\ \{5.0 1.0 0.5 \textbf{0.1}\} \\ feature size: \textbf{50}}\\
        & & SI & 0.001 & 0 & 0 & \\
        & & UW & 0.001 & 0 & 0 & \\
        \bottomrule
    \end{tabular}
    \label{tab:mnist_CL_RR_hyper}
\end{table}

\begin{table}[ht]
    \small
    \centering
    \caption{Hyperparameters for MNIST - RL \& RR.}
    \begin{tabular}{cclcccc}
        \toprule
        \multirow{2}{*}{\centering Dataset} & \multirow{2}{*}{\centering Tasks} & \multirow{2}{*}{SMTO} & \multicolumn{3}{c}{Hyperparameters} & \multirow{2}{*}{\centering SMTO specific parameters} \\
        \cmidrule{4-6}
        & & & LR & P & WD & \\
        \midrule
         MNIST & \makecell{RL \\ RR} & Single Task & \makecell{0.0025 \\ 0.0025} & 0 & 0 & \\
         \midrule
        & & Unit. Scal. & 0.001 & 0 & 0 & \\
        & & Auto-Lambda & 0.001 & 0 & 0 & \makecell{auxiliary learning rate scale: \\ \{1000 \textbf{100} 10 1 0.01 0.001 0.0001\} \\ initial weight: \textbf{0.1}} \\
        & & CAGrad & 0.0025 & 0 & 0 & $c$:  \{0.2, \textbf{0.5}, 0.8\} \\
        & & CDTT & 0.001 & 0 & 0 & $\alpha$: \{0.2, 0.4, 0.6, 0.8, \textbf{1.0}\} \\
        & & EDM & 0.001 & 0 & 0 & \\
        & & FAMO & 0.001 & 0 & 0 & \makecell{$\gamma$: \{0.01, 0.001, \textbf{0.0001}\} \\ learning rate of the task logits: \textbf{0.025}} \\
        & & IMTL & 0.001 & 0 & 0 & \\
        & & Nash-MTL & 0.001 & 0 & 0 & optimizer iterations: \textbf{20}\\
        & & RotoGrad & 0.001 & 0 & 0 & \makecell{auxiliary learning rate scale: \\ \{5.0 1.0 0.5 \textbf{0.1}\} \\ feature size: \textbf{50}} \\
        \bottomrule
        \end{tabular}
    \label{tab:mnist_RL_RR_hyper}
\end{table}

\begin{table}[ht]
    \small
    \centering
    \caption{Hyperparameters for Cityscapes.}
    \begin{tabular}{cclcccc}
        \toprule
        \multirow{2}{*}{\centering Dataset} & \multirow{2}{*}{\centering Tasks} & \multirow{2}{*}{\centering SMTO} & \multicolumn{3}{c}{Hyperparameters} & \multirow{2}{*}{\centering SMTO specific parameters} \\
        \cmidrule{4-6}
        & & & LR & P & WD & \\
        \midrule
        City. & \makecell{S \\ \\ D \\ \\ I} & Single Task & \makecell{ACC: 0.001 \\ mIoU: 0.001 \\ L1 abs: 0.001 \\ L1 rel: 0.0005 \\ L1 abs: 0.001} & \makecell{0 \\ 0 \\ 0 \\ 0 \\ 0} & \makecell{$1e-4$ \\ $1e-4$ \\ $1e-4$ \\ $1e-4$ \\ $1e-5$} & \\
        \midrule
        & & Unit. Scal. & 0.001 & 0 & $1e-4$ & \\
        & & Auto-Lambda & 0.001 & 0 & $1e-4$ &  \makecell{auxiliary learning rate scale: \\ \{100 10 1 0.1 0.01 0.001 \textbf{0.0001}\} \\ initial weight: \textbf{0.1}} \\
        & & CAGrad & 0.001 & 0 & $1e-4$ & $c$: \{\textbf{0.2}, 0.5, 0.8\} \\
        & & CDTT & 0.001 & 0 & $1e-4$ & $\alpha$: \{\textbf{0.2}, 0.4, 0.6, 0.8, 1.0\} \\
        & & EDM & 0.001 & 0 & $1e-4$ & \\
        & & FAMO & 0.001 & 0 & $1e-4$ & \makecell{$\gamma$: \{\textbf{0.01}, 0.001, 0.0001\} \\ learning rate of the task logits: \textbf{0.025}} \\
        & & IMTL & 0.001 & 0 & $1e-4$ & \\
        & & Nash-MTL & 0.001 & 0 & $1e-4$ & optimizer iterations: \textbf{20}\\
        & & RotoGrad & 0.001 & 0 & 0 & \makecell{auxiliary learning rate scale: \\ \{5.0 \textbf{1.0} 0.5 0.1\} \\ feature size: \textbf{512}} \\
        \bottomrule
        \end{tabular}
    \label{tab:city_hyper}
\end{table}

\begin{table}[ht]
    \small
    \centering
    \caption{Hyperparameters for QM9.}
    \begin{tabular}{cclcccc}
        \toprule
        \multirow{2}{*}{\centering Dataset} & \multirow{2}{*}{\centering Tasks} & \multirow{2}{*}{\centering SMTO} & \multicolumn{3}{c}{Hyperparameters} & \multirow{2}{*}{\centering SMTO specific parameters} \\
        \cmidrule{4-6}
        & & & LR & P & WD & \\
        \midrule
        QM9 & \makecell{$C_v$ \\ $G$ \\ $H$ \\ $U_0$ \\ $U$ \\ $\alpha$ \\ $\epsilon_{homo}$ \\ $\epsilon_{lumo}$ \\ $\mu$ \\ $R^2$ \\ ZPVE} & Single Task & \makecell{0.001 \\ 0.001 \\ 0.001 \\ 0.0005 \\ 0.001 \\ 0.005 \\ 0.001 \\ 0.001 \\ 0.001 \\ 0.001 \\ 0.001} & \makecell{0 \\ 0 \\ 0 \\ 0 \\ 0 \\ 0 \\ 0 \\ 0 \\ 0 \\ 0 \\ 0} & \makecell{0 \\ 0 \\ 0 \\ 0 \\ 0 \\ 0 \\ 0 \\ 0 \\ 0 \\ 0 \\ 0} & \\
        \midrule
        & & Unit. Scal. & 0.005 & 0 & 0 & \\
        & & Auto-Lambda & 0.001 & 0 & 0 & \makecell{auxiliary learning rate scale: \\ \{100, \textbf{10}, 1.0\} \\ initial weight: \textbf{0.1}} \\
        & & CAGrad & 0.001 & 0 & 0 & $c$: \{0.2, 0.5, \textbf{0.8}\} \\
        & & CDTT & 0.001 & 0 & 0 & $\alpha$: \{0.2, 0.4, 0.6, 0.8, \textbf{1.0}\} \\
        & & EDM & 0.001 & 0 & 0 & \\
        & & FAMO & 0.005 & 0 & 0 & \makecell{$\gamma$: \{0.01, 0.001, \textbf{0.0001}\} \\ learning rate of the task logits: \textbf{0.025}} \\
        & & Nash-MTL & 0.001 & 0 & 0 & optimizer iterations: \textbf{20}\\
        \bottomrule
        \end{tabular}
    \label{tab:QM9_hyper}
\end{table}

\end{document}